\documentclass[10pt,letterpaper]{article}

\usepackage[letterpaper,margin=0.72in]{geometry}
\usepackage[T1]{fontenc}
\usepackage[utf8]{inputenc}
\usepackage{amsmath,amssymb,amsthm}
\usepackage{booktabs}
\usepackage{graphicx}
\usepackage{placeins}
\usepackage{microtype}
\usepackage{times}
\usepackage{tabularx}
\usepackage{xcolor}
\usepackage[round,authoryear]{natbib}
\usepackage{xurl}
\usepackage{hyperref}
\hypersetup{hidelinks}
\graphicspath{{figures/}}

\newtheorem{theorem}{Theorem}
\newtheorem{proposition}{Proposition}

\newcommand{\R}{\mathbb{R}}

\newcommand{\calG}{\mathcal{G}}

\newcommand{\expe}{\mathrm{e}}

\title{Jordan-RoPE: Non-Semisimple Relative Positional Encoding via Complex Jordan Blocks}
\author{
Yaobo Zhang\\
School of Physics, Ningxia University\\
Yinchuan, 750021, China\\
\texttt{yaobozhang@zju.edu.cn}
}
\date{}

\newenvironment{wideabstract}{%
  \begin{center}\bfseries Abstract\end{center}%
  \small\noindent
}{\par}

\begin{document}
\twocolumn[
\maketitle

\begin{wideabstract}
Relative positional encodings determine the primitive functions of query-key lag that are available to attention.
We study a non-semisimple rotary extension in which a complex RoPE eigenvalue and a nilpotent response live in the same defective Jordan block.
The resulting exact one-parameter representation turns RoPE into a finite frequency-jet basis: the order-two block supplies the first tangent mode \(d\,\expe^{i\omega d}\), while order-\(m\) chains supply modes \(d^r\expe^{i\omega d}\) for \(r=0,\ldots,m-1\).
Geometrically, complex Jordan blocks turn rotary frequencies into infinitesimal frequency jets rather than adding an independent distance axis.
Because these blocks are non-orthogonal, we spell out the contragredient query action required to recover a pure relative score, and we separate Exact/raw and Scaled-exact Jordan-RoPE representations from Stabilized Jordan-RoPE approximations.
Fixed-kernel probes and a Jordan-friendly synthetic language-model task confirm that the coupled basis is useful when the target contains distance-modulated phase interactions.
High-order probes verify the finite-jet hierarchy, while MusicNet and WikiText experiments are treated as task-selective and boundary evidence rather than broad performance claims: MusicNet gives a natural positive case for a Scaled-exact order-three variant, and small WikiText runs suggest that weak jet corrections help mainly when composed with a recency bias such as ALiBi.
\end{wideabstract}
\vspace{1em}
]

\section{Introduction}

Attention uses inner products to decide which previous representations should influence the current token \citep{vaswani2017attention}.
Without positional modification, however, this inner product has no intrinsic notion of order.
A relative positional encoding therefore does more than label tokens: it specifies which functions of the query-key lag can enter the primitive attention logit \citep{shaw2018self,dai2019transformerxl}.
This point becomes especially important in long-context modeling, where the same positional rule is evaluated far outside the range of lags seen during training.

Many successful mechanisms can be understood from this perspective.
Relative position representations make attention depend explicitly on pairwise displacement.
RoPE turns absolute rotations of queries and keys into a relative phase \citep{su2021roformer}.
ALiBi adds a monotone distance penalty directly to the logits \citep{press2022alibi}.
A later line of work modifies, interpolates, or rescales rotary phases to improve length extrapolation \citep{sun2022length,chen2023position,peng2023yarn,ding2024longrope}.
Although these methods differ operationally, they all make a choice about the analytic form of the relative-position kernel.

A complementary viewpoint is algebraic \citep{kogkalidis2023algebraic,ostmeier2024liere,puranik2026group,zhang2025grape}.
If the positional modification is linear, translation-invariant, and continuous, then the relative operator is naturally a one-parameter matrix action.
This observation explains why rotations, exponential decays, and additive distance biases appear so persistently.
It also exposes a less explored possibility: the generator need not be diagonalizable.
In a defective block, the matrix exponential contains nilpotent polynomial factors.
For real eigenvalues this gives linear or higher-order shear.
For complex eigenvalues it gives something more structured: a rotary phase coupled to a polynomial response.

\paragraph{Motivating observation.}
The starting point of this note is a small but suggestive gap in the group-theoretic view of positional encoding.
A recent group-theoretic discussion observed that, under linearity, translation invariance, and continuity, relative positional encodings are naturally organized by one-parameter matrix groups \citep{puranik2026group}.
In that classification, diagonalizable generators lead to familiar exponential, damping, and rotary mechanisms, while defective generators lead to polynomial factors through Jordan blocks.
The real defective case already gives a linear-in-time shear and can be related to ALiBi after augmenting the bilinear form.
The complex defective case, however, is less developed: the simplest real realization involves a doubled conjugate pair and a nontrivial Jordan chain.

This observation suggests a different interpretation of defective positional encoding.
Rather than treating it as a pathological matrix case, we view it as a non-normal linear dynamical system inserted into the attention kernel.
The nilpotent part of a Jordan block does not merely add a distance feature; when coupled to a complex rotary eigenvalue, it generates the tangent mode \(d\,\expe^{i\omega d}\).
The question studied here is therefore deliberately narrow: what new primitive relative-position features arise from the simplest complex Jordan block, and can a Transformer use them when the task rewards distance-modulated phase structure?

This note follows that non-semisimple corner.
We place the rotary complex eigenvalue and a nilpotent response in the same defective Jordan block.
The resulting relative operator contains primitive oscillatory-polynomial features such as
\[
d\,\expe^{-\gamma d}\cos(\omega d),\qquad d\,\expe^{-\gamma d}\sin(\omega d),
\]
which realize a distance-modulated phase basis at the pre-softmax bilinear-kernel level.
Equivalently, the nilpotent channel supplies the frequency-tangent feature
\[
d\,\expe^{i\omega d}=\frac{1}{i}\partial_\omega \expe^{i\omega d}.
\]
The order-two construction is therefore the first nontrivial member of a finite frequency-jet hierarchy: higher-order complex Jordan chains provide primitive modes \(d^r\expe^{i\omega d}\).

\paragraph{Geometric intuition: infinitesimal thickening of rotary frequencies.}
RoPE can be viewed as sampling points on the Fourier-character curve
\[
\omega\mapsto \chi_\omega,\qquad \chi_\omega(d)=\expe^{i\omega d}.
\]
A complex Jordan block does not attach an independent distance coordinate to this curve.
Instead, it replaces a reduced frequency point by a nilpotent thickening.
Formally, if \(\varepsilon^m=0\), then evaluating the Fourier character at an infinitesimal neighborhood of \(\omega\) gives
\[
\expe^{i(\omega+\eta\varepsilon)d}
=
\expe^{i\omega d}
\sum_{r=0}^{m-1}\frac{(i\eta d)^r}{r!}\varepsilon^r .
\]
Up to a harmless rescaling of the nilpotent coordinate, this is the same finite chain as the undamped Jordan operator
\[
G_m^{(\gamma=0)}(d)
=
\expe^{i\omega d}
\sum_{r=0}^{m-1}\frac{(\eta d)^r}{r!}N_m^r .
\]
Thus order-\(m\) Jordan-RoPE realizes the \((m-1)\)-jet of the Fourier character at \(\omega\).
In this sense, RoPE samples reduced spectral points, while Jordan-RoPE samples non-reduced infinitesimal neighborhoods.
The direct-sum baseline adds an external distance axis; the Jordan block instead thickens the rotary frequency itself.

We do not introduce group-theoretic positional encoding as a new paradigm, nor do we claim that nilpotent positional mechanisms are new.
Instead, we ask what happens when the rotary complex eigenvalue and the nilpotent response are coupled inside one defective block.
This distinction matters: a large downstream network may learn products of separate channels, but the positional encoding itself supplies a different primitive logit basis.

Our contributions are structural and experimental:
\begin{itemize}
    \item We identify the order-two complex Jordan block as the minimal non-semisimple extension of rotary relative positional encoding in which phase and nilpotent response are coupled rather than placed in separate subspaces.
    \item We formulate Exact Jordan-RoPE as a one-parameter relative representation and give its real block implementation. Because the block is non-orthogonal, we spell out the contragredient query action needed to recover a pure relative attention score.
    \item We separate the exact representation from Stabilized Jordan-RoPE. The bounded-\(\tau\) implementation preserves the local Jordan response but breaks the group law, while Scaled-exact Jordan-RoPE preserves the one-parameter structure with controlled damping.
    \item We test the induced basis on fixed kernel probes, structured sequence probes, a synthetic language-model task with teacher kernel \(K(d)=(d/L)\cos(\omega d)\), and a small WikiText-103 byte language model.
    \item We add a high-order frequency-jet diagnostic showing that order-\(m\) complex Jordan blocks realize the finite jet basis \(\{d^r\expe^{i\omega d}\}_{r=0}^{m-1}\). Beyond fixed-basis and synthetic high-jet probes, a small music-token transfer sweep identifies MusicNet as a task-selective natural positive case for a Scaled-exact order-three variant, and auxiliary WikiText runs show that normalized high-jet spectra help mainly as weak corrections composed with ALiBi rather than as standalone replacements for recency bias.
\end{itemize}

The evidence should be read structurally rather than as a broad performance claim.
At the representation level, Exact Jordan-RoPE is a non-semisimple one-parameter relative positional representation when implemented with the contragredient query action.
At the basis level, it directly provides the primitive tangent feature \(d\,\expe^{i\omega d}\), instead of separating phase and distance into independent channels.
At the task level, when the teacher kernel contains \((d/L)\cos(\omega d)\), Stabilized Jordan-RoPE extrapolates better than RoPE and direct-sum baselines.
Natural language modeling is treated as a boundary test rather than the main evidence.

\section{Related Work and Positioning}

\begin{figure*}[t]
    \centering
    \includegraphics[width=\linewidth]{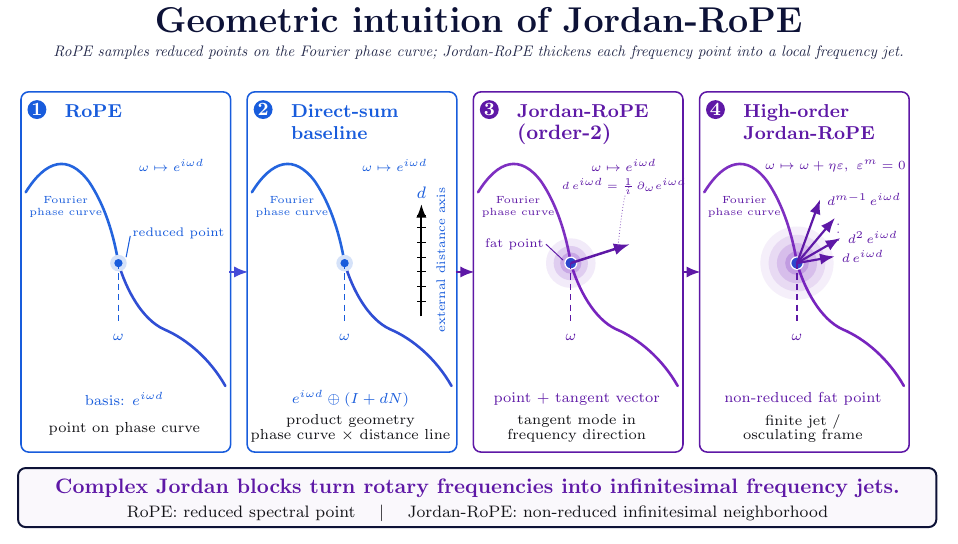}
    \caption{Geometric intuition of Jordan-RoPE. RoPE samples reduced points on the Fourier-character curve. A direct-sum baseline adds an external distance coordinate. A complex Jordan block instead replaces a rotary frequency by a non-reduced infinitesimal neighborhood, producing the tangent mode \(d\,\expe^{i\omega d}\) and, at order \(m\), the finite jet basis \(\{d^r\expe^{i\omega d}\}_{r=0}^{m-1}\).}
    \label{fig:method-schematic}
\end{figure*}

Let \(i\) be a query position and \(j\le i\) a key position in causal attention.
We write the causal lag as
\[
d=i-j\ge 0.
\]
This convention is important: damping terms are written as \(\expe^{-\gamma d}\), so reversing the sign would make distant past tokens grow rather than decay.
The sign of the oscillatory phase can be absorbed by \(\omega\mapsto -\omega\), but the damping sign cannot.

\paragraph{Relative positional encoding.}
Relative positional representations make attention scores depend on pairwise displacement rather than only absolute token index \citep{shaw2018self}.
Transformer-XL \citep{dai2019transformerxl} combines segment-level recurrence with relative positional terms to model dependencies beyond a fixed segment.
RoPE applies orthogonal block rotations to queries and keys so that \(R_i^\top R_j=R_{j-i}\), turning absolute rotations into a relative phase \citep{su2021roformer}.
ALiBi adds a head-dependent linear penalty, typically \(-m_h d\), directly to the logits \citep{press2022alibi}.

\paragraph{Length extrapolation and context extension.}
Long-context evaluation asks positional kernels to operate at lags larger than those seen during training.
XPos modifies rotary attention with a decay structure for length extrapolation \citep{sun2022length}.
Position Interpolation rescales positions into the pretraining range for RoPE-based LLMs \citep{chen2023position}, and YaRN improves the efficiency of such RoPE context extension \citep{peng2023yarn}.
LongRoPE further uses non-uniform interpolation and progressive extension to reach much larger context windows \citep{ding2024longrope}.
These methods differ operationally, but share the concern that unseen-distance behavior is shaped by the analytic form of the positional kernel.

There is also a kernel/function-design view of relative position encodings.
KERPLE treats relative positional embedding as a kernelization of positional differences, while FIRE parameterizes relative position as a learned function and interpolates it for long contexts \citep{chi2022kerple,li2024fire}.
Empirical studies such as NoPE further show that length generalization is highly sensitive to the way positional information is supplied or omitted \citep{kazemnejad2023impact}.
Our construction fits this line at the primitive-basis level: the lag functions are fixed by a non-semisimple one-parameter representation.

\paragraph{Algebraic and group-action views.}
Algebraic Positional Encodings map algebraic structure into orthogonal operators \citep{kogkalidis2023algebraic}.
LieRE generalizes RoPE by learning higher-dimensional skew-symmetric generators and exponentiating them to rotations \citep{ostmeier2024liere}.
Jane Street's group-theoretic treatment emphasizes one-parameter matrix actions for linear, translation-invariant positional mechanisms and notes that defective generators produce polynomial terms \citep{puranik2026group}.
GRAPE gives a broad group-action framework that recovers RoPE from rotational actions and ALiBi-style biases from unipotent additive lifts \citep{zhang2025grape}.

\paragraph{From semisimple positional geometry to non-semisimple blocks.}
Most algebraic or group-theoretic positional encodings used in attention are built from well-conditioned actions such as rotations, orthogonal maps, or direct sums of independently acting subspaces.
Our object is narrower but different.
We do not introduce group-theoretic positional encoding as a new paradigm.
Instead, we study what happens when the rotary complex eigenvalue is placed in a defective block, so that the phase factor and the nilpotent response are coupled in the same non-semisimple representation.
This produces primitive bilinear features of the form \(d\,\expe^{i\omega d}\), rather than separate access to \(\expe^{i\omega d}\) and \(d\).
We therefore do not claim novelty at the level of group-theoretic positional encoding, nilpotent positional mechanisms, or long-context extrapolation in general.
The specific contribution is the coupled complex Jordan block and the induced primitive basis \(d\,\expe^{i\omega d}\) for relative attention logits.
Recent theory on linear recurrences and structured state-space models also suggests that complex eigenvalues are not merely a notational convenience: they can improve information storage and yield formal expressivity or learnability gaps relative to real parameterizations \citep{orvieto2024universality,ranmilo2024complex}.
Jordan-RoPE asks a complementary question: what happens when the complex rotary eigenvalue is made defective rather than diagonal?
Figure~\ref{fig:method-schematic} summarizes this basis compared with RoPE and a direct-sum construction.

\section{Non-semisimple Relative Representations}
\label{sec:representations}

\subsection{Relative scoring for non-orthogonal actions}

We use the causal-lag convention \(d=i-j\ge 0\).
The representation-level object in this paper is an exact one-parameter relative operator
\[
G(d)=\exp(dJ),\qquad G(d_1+d_2)=G(d_1)G(d_2).
\]
For orthogonal rotary actions, applying the same absolute map to queries and keys is enough.
For non-orthogonal actions, this is no longer true: in general \(G(i)^\top G(j)\ne G(i-j)\).
The relative score is recovered by using the dual action on queries.

\begin{proposition}[Relative scoring for non-orthogonal actions]
Let \(G:\R\to \mathrm{GL}(D)\) be a one-parameter representation, and let \(A(t)=G(-t)\).
Define transformed queries and keys by
\[
\widetilde q_i=A(i)^{-\top}q_i,\qquad
\widetilde k_j=A(j)k_j.
\]
Then the attention score is a function of the causal lag:
\[
\langle \widetilde q_i,\widetilde k_j\rangle
=q_i^\top A(i)^{-1}A(j)k_j
=q_i^\top G(i-j)k_j.
\]
\end{proposition}

\begin{proof}
Since \(A(t)=G(-t)\), the one-parameter law gives
\[
A(i)^{-1}A(j)=G(-i)^{-1}G(-j)=G(i)G(-j)=G(i-j).
\]
Substituting this into the bilinear score gives the claim.
\end{proof}

This contragredient query action is the main implementation distinction from ordinary RoPE.
It is the convention used by the Transformer experiments in this repository: the module applies a position-wise transform with \(t=-\mathrm{position}\), using the inverse transpose on queries and the primal transform on keys.

\subsection{Semisimple, direct-sum, and defective generators}

RoPE is the semisimple complex case: for one rotary frequency, the generator is \(J_{\mathrm{RoPE}}=i\omega I\), giving the phase basis \(\expe^{i\omega d}\).
A direct-sum phase/distance baseline keeps these ingredients in separate invariant subspaces, schematically
\[
J_{\mathrm{ds}}=
\begin{pmatrix}
i\omega I & 0\\
0 & N
\end{pmatrix},
\qquad
\exp(dJ_{\mathrm{ds}})
=\expe^{i\omega d}\oplus(I+dN).
\]
This makes phase and distance available, but not as a coupled primitive bilinear feature.

The non-semisimple rotary object studied here instead places the complex rotary eigenvalue and the nilpotent response in the same indecomposable block:
\[
J_m=(-\gamma+i\omega)I+\eta N_m,\qquad N_m^m=0.
\]
The phase and polynomial response are therefore coupled before the attention softmax.
Several baselines arise by constraining this object.
If \(\eta=0\), it becomes damped RoPE.
If \(\omega=0\), it becomes a real unipotent/Jordan distance block.
If \(\gamma=0\), it is an undamped complex Jordan block.

\begin{table*}[t]
    \centering
    \caption{Mechanism taxonomy at the primitive pre-softmax bilinear-kernel level. Complex Jordan blocks are the non-semisimple exact case: they thicken the rotary frequency itself rather than adding a separate distance axis.}
    \label{tab:mechanisms}
    \small
    \begin{tabularx}{\linewidth}{llXX}
        \toprule
        Class & Method & Primitive basis & Representation status \\
        \midrule
        Semisimple & RoPE & \(\cos(\omega d),\sin(\omega d)\) & \textsc{Exact} \\
        Augmented & ALiBi-like lift & \(1,d\) after a constant-channel lift & \textsc{Augmented exact} \\
        Direct sum & RoPE \(\oplus\) distance & separate phase and distance channels & \textsc{Exact} \\
        \textbf{Non-semisimple} & \textbf{Complex Jordan} & \textbf{coupled jets \(d^r\expe^{i\omega d}\)} & \textbf{\textsc{Exact}} \\
        Non-semisimple & Scaled-exact Jordan-RoPE & scaled coupled jets \( (d/L)^r\expe^{-cd/L}\expe^{i\omega d}\) & \textsc{Exact} \\
        Stabilized & Bounded-\(\tau\) shear & bounded local first-jet response & \textsc{Approx.} \\
        Composition & Jordan + ALiBi & jets plus additive recency bias & auxiliary \\
        \bottomrule
    \end{tabularx}
\end{table*}

For the ALiBi-like row, ``augmentation'' means that a linear \(d\) term can be realized after adding a constant channel to the bilinear form, for example through the unipotent lift
\[
\begin{pmatrix}1&d\\0&1\end{pmatrix},
\]
rather than claiming that the usual ALiBi implementation is itself a query/key linear action.

\subsection{Complex Jordan blocks and finite frequency jets}
\label{sec:high-order}

The order-\(m\) complex Jordan relative operator is
\[
G_m(d)
=\expe^{dJ_m}
=\expe^{(-\gamma+i\omega)d}
\sum_{r=0}^{m-1}\frac{(\eta d)^r}{r!}N_m^r.
\]
This is the standard matrix-function behavior of a Jordan block: applying an analytic function to a nontrivial Jordan chain exposes derivatives of that function along the nilpotent direction \citep{higham2008functions}.
For \(f_\omega(d)=\expe^{i\omega d}\),
\[
\partial_\omega^r\expe^{i\omega d}=(id)^r\expe^{i\omega d}.
\]

\begin{theorem}[Finite rotary frequency-jet basis]
An order-\(m\) complex Jordan block generates primitive relative-position features
\[
d^r\expe^{-\gamma d}\cos(\omega d),\qquad
d^r\expe^{-\gamma d}\sin(\omega d),
\qquad r=0,\ldots,m-1.
\]
Equivalently, up to constants, it realizes the finite frequency-jet span
\[
\operatorname{span}\{\partial_\omega^r\expe^{i\omega d}\}_{r=0}^{m-1}.
\]
\end{theorem}

\begin{proof}
Expanding the matrix exponential gives
\[
\expe^{d((-\gamma+i\omega)I+\eta N_m)}
=\expe^{(-\gamma+i\omega)d}
\sum_{r=0}^{m-1}\frac{(\eta d)^r}{r!}N_m^r,
\]
because \(N_m^m=0\).
Realifying the complex coordinates splits each complex feature into cosine and sine components, yielding the displayed basis.
\end{proof}

Thus an order-three block supplies phase, first-jet, and second-jet channels, while an order-four block also supplies a third-jet channel.
The hierarchy is nested: an order-\(m\) chain contains all lower-order jets in the same frequency family.
This is why a higher-order basis can fit a lower-jet target on a finite grid without implying that the higher order is theoretically necessary.

\subsection{The order-two block as the first rotary tangent}

The minimal nontrivial member is \(m=2\).
Let
\[
N=\begin{pmatrix}0&1\\0&0\end{pmatrix},\qquad N^2=0,
\]
and
\[
J_2=(-\gamma+i\omega)I+\eta N.
\]
Then
\[
G_2(d)
=\expe^{(-\gamma+i\omega)d}(I+\eta dN)
=\expe^{(-\gamma+i\omega)d}
\begin{pmatrix}
1&\eta d\\
0&1
\end{pmatrix}.
\]
The resulting real primitive features are
\[
\begin{aligned}
&\expe^{-\gamma d}\cos(\omega d),\quad
\expe^{-\gamma d}\sin(\omega d),\\
&d\,\expe^{-\gamma d}\cos(\omega d),\quad
d\,\expe^{-\gamma d}\sin(\omega d).
\end{aligned}
\]
Order-two Jordan-RoPE is therefore the first rotary frequency tangent:
\[
d\,\expe^{i\omega d}=\frac{1}{i}\partial_\omega\expe^{i\omega d}.
\]
It should not be read as merely adding a separate distance channel to RoPE; the nilpotent response is coupled to the rotary phase inside one defective block.

\subsection{Real block form and channel budget}

Writing
\[
R_\phi=\begin{pmatrix}\cos\phi&-\sin\phi\\ \sin\phi&\cos\phi\end{pmatrix},
\]
the realified order-two block acts on four real coordinates:
\[
\calG_2(d)
=\expe^{-\gamma d}
\begin{pmatrix}
R_{\omega d} & \eta d\,R_{\omega d}\\
0 & R_{\omega d}
\end{pmatrix}.
\]
An order-\(m\) complex Jordan block realifies to \(2m\) real coordinates per frequency.
Thus an order-two block uses \(d_h/4\) frequencies in a head of dimension \(d_h\), an order-three block uses \(d_h/6\), and an order-four block uses \(d_h/8\), assuming divisibility.
The implementation repeats such blocks across a RoPE-style frequency grid; for order two, the default is
\[
\omega_k=\theta^{-2k/d_h}.
\]
The reported comparisons keep the same total head dimension and attention-head budget across methods.
They are therefore dimension-matched, but not exactly frequency-count matched to standard RoPE, whose two-dimensional rotary blocks can use \(d_h/2\) frequencies.
A frequency-grid ablation, for example using \(\omega_k=\theta^{-4k/d_h}\) to span the same endpoint range with \(d_h/4\) Jordan frequencies or restricting RoPE to the same number of frequencies, is a natural follow-up and is not included in the present runs.

\section{Implementation Regimes}
\label{sec:implementation-regimes}

\subsection{Raw exact Jordan-RoPE}

The Exact/raw order-\(m\) object is
\[
G_{m,\mathrm{raw}}(d)
=\expe^{(-\gamma+i\omega)d}
\sum_{r=0}^{m-1}\frac{(\eta d)^r}{r!}N_m^r.
\]
For \(m=2\), this reduces to
\[
G_{2,\mathrm{raw}}(d)
=\expe^{(-\gamma+i\omega)d}(I+\eta dN).
\]
It is a strict one-parameter representation, so the contragredient position-wise implementation recovers a pure relative score.
Its drawback is scale: the nilpotent channels grow polynomially with \(d\), which is useful for unbounded synthetic targets but can be numerically or statistically awkward at long context.

\subsection{Scaled-exact Jordan-RoPE}

The scaled-exact family keeps the representation law while normalizing damping and shear by the training length \(L\):
\[
J_{m,\mathrm{scaled}}
=\left(-\frac{c}{L}+i\omega\right)I+\frac{\eta}{L}N_m,
\]
and therefore
\[
G_{m,\mathrm{scaled}}(d)
=\expe^{-cd/L}\expe^{i\omega d}
\sum_{r=0}^{m-1}\frac{(\eta d/L)^r}{r!}N_m^r .
\]
This is not an approximation to a group action; it is an exact one-parameter representation with a different generator.
Writing \(x=d/L\), the real primitive span contains, up to constant coefficients,
\[
x^r\expe^{-cx}\cos(\omega d),\qquad
x^r\expe^{-cx}\sin(\omega d),
\qquad r=0,\ldots,m-1.
\]
This is the origin of the scaled high-jet targets used in the experiments: the factors \(x^2\expe^{-0.1x}\cos(\omega d)\) and \(x^3\expe^{-0.1x}\cos(\omega d)\) are the second and third real frequency jets of the scaled order-\(m\) relative operator.
In the experiments, scaled-exact variants are important because they preserve the representation-level theory while controlling long-lag growth.

\subsection{Stabilized Jordan-RoPE}

Stabilized Jordan-RoPE replaces the linear shear coordinate by the bounded-\(\tau\) coordinate
\[
\tau(d)=\frac{d}{1+d/L}.
\]
The corresponding explicit relative-kernel diagnostic is
\[
G_{\mathrm{stab}}(d)
=\expe^{-\gamma d}
\begin{pmatrix}
R_{\omega d} & \eta\tau(d)R_{\omega d}\\
0 & R_{\omega d}
\end{pmatrix}.
\]
This keeps the same first-order response near \(d=0\), since \(\tau(d)=d+O(d^2/L)\), while preventing the nilpotent channel from growing without bound.
However,
\[
\tau(d_1+d_2)\ne \tau(d_1)+\tau(d_2),
\]
so Stabilized Jordan-RoPE is not a strict one-parameter representation.
It is a representation-inspired engineering approximation.

\subsection{Position-wise q/k factorization versus explicit relative kernels}

The basis-level experiments can evaluate an explicit pairwise kernel \(G(d)\) directly.
Transformer attention is implemented more efficiently by position-wise query/key maps.
For exact one-parameter representations, Proposition~1 shows that the two views coincide.
For nonlinear bounded-\(\tau\) coordinates, they generally do not.

To see the obstruction, isolate the unipotent coordinate and write
\[
B_\sigma(t)=I+\eta\sigma(t)N.
\]
Then
\[
B_\sigma(-i)^{-1}B_\sigma(-j)
=I+\eta\bigl(\sigma(-j)-\sigma(-i)\bigr)N.
\]
If \(\sigma(t)=t\), this coefficient is \(\eta(i-j)\), a pure function of the lag.
If \(\sigma\) is nonlinear, such as the bounded-\(\tau\) coordinate, the coefficient generally depends on the absolute pair \((i,j)\), not only on \(i-j\).
This is why Exact/raw and Scaled-exact Jordan-RoPE variants are exact relative representations, while Stabilized Jordan-RoPE should be understood as an engineering approximation with the same local Jordan response.

\subsection{Practical parameterization and auxiliary compositions}

The implementation parameterizes damping and nilpotent amplitude with constrained learnable variables, using nonnegative damping and bounded \(\eta\).
The experimental names follow the hierarchy above:
\textbf{Exact/raw Jordan-RoPE} disables bounded \(\tau\),
\textbf{Scaled-exact Jordan-RoPE} uses the \(d/L\)-normalized generator and preserves the group law, and
\textbf{Stabilized Jordan-RoPE} refers to the bounded-\(\tau\) implementation.
The ALiBi-composed variants add the usual additive recency bias to the attention logits; they are auxiliary compositions, not new one-parameter representations.

For mix-order jet experiments, we also use normalized jet spectra.
There the nonnegative coefficients \(\alpha_r\) sum to one and specify the distribution of mass across the base phase and higher jet orders.
This separates coefficient distribution from overall gain: a spectrum such as \((.70,.20,.10)\) means that most mass remains on the base rotary phase, with weaker first- and second-jet corrections.

The high-jet construction is adjacent to polynomial-basis memory models: HiPPO represents sequence history through online projections onto polynomial bases, and later work interprets S4 through exponentially warped orthogonal bases \citep{gu2020hippo,gu2022hippo}.
Our mechanism is different.
The polynomial factors are not used to compress history into a recurrent state; they are inserted directly into the relative attention logit basis as polynomially weighted Fourier modes.
In this sense, high-order Jordan-RoPE asks whether a Transformer benefits when the pre-softmax positional bilinear form itself contains a finite frequency-jet hierarchy.

\FloatBarrier

\section{Experiments}
\label{sec:experiments}

\subsection{Experimental setup}

Unless otherwise noted, reported means and standard deviations use three seeds.
The main order-two synthetic query-LM trains at length \(1024\) and evaluates at \(1024\), \(2048\), \(4096\), and \(8192\).
The high-order synthetic and music-token transfer runs are smaller auxiliary matrices trained at length \(256\), and the auxiliary high-jet WikiText run is trained at length \(512\).
The natural language experiment is a byte-level WikiText-103 \citep{merity2017pointer} language model.
The plots and tables in the experimental sections are generated from local CSV logs produced by the released scripts and configs, including supplemental exact-scaled runs.
The reported experimental names follow the convention above: \textbf{Stabilized Jordan-RoPE} is the bounded-\(\tau\) default, \textbf{Exact/raw Jordan-RoPE} disables bounded \(\tau\), and \textbf{Scaled-exact Jordan-RoPE} denotes the \(d/L\)-normalized exact variant.
All Transformer baselines use the same total head dimension \(d_h\).
RoPE can therefore use \(d_h/2\) two-dimensional rotary frequencies, while the order-two complex Jordan variants use \(d_h/4\) four-dimensional blocks; the direct-sum baseline splits the same \(d_h\) between a RoPE subspace and a real unipotent/Jordan distance subspace.
Damped RoPE in these comparisons is implemented as the same order-two Jordan block with \(\eta=0\), so it also uses \(d_h/4\) frequencies; a full-frequency damped RoPE baseline is left to a separate ablation.
Thus the comparisons are dimension-matched under the implemented attention budget, while frequency-count matching is left as a separate ablation.

\begin{table*}[t]
    \centering
    \caption{Experiment roadmap. The evidence chain moves from primitive basis checks, to trainable synthetic tasks, to high-order jets, to task-selective natural transfer and natural-language boundary tests.}
    \label{tab:experiment-map}
    \small
    \begin{tabularx}{\linewidth}{lccX}
\toprule
Claim tested & Experiment group & Lengths & Main takeaway \\
\midrule
Primitive basis & Fixed kernel and structured probes & \(1024\to8192\) & Jordan features directly fit mixed distance-phase targets. \\
Usable inductive bias & Synthetic query LM & \(1024\to8192\) & A trained Transformer can exploit the coupled first-jet mode. \\
Finite jet hierarchy & High-order fixed-basis probes & \(1024\to8192\) & Order-\(m\) blocks fit frequency jets up to order \(m-1\). \\
Trainable high jets & High-jet synthetic query-LM & \(256\to1024\) & Exact order-three/four variants help when the teacher contains higher jets. \\
Natural positive case & Music-token transfer & \(256\to2048\) & MusicNet gives a task-selective Scaled-exact \(m=3\) positive case. \\
Boundary & WikiText boundary tests & \(512/1024\to8192\) & ALiBi remains strong; Jordan/jet terms help mainly as weak corrections. \\
\bottomrule
\end{tabularx}

\end{table*}

\subsection{Basis-level mixed target}

The first diagnostic is a linear readout over fixed relative-position features.
This is deliberately a kernel-level test, not a general neural approximation test.
The mixed target is
\[
y(d)=\frac{d}{L}\cos(\omega d),
\]
fit on \(d<L=1024\) and evaluated out to \(8192\).
At this primitive logit-basis level, additive or direct-sum combinations provide separate phase and distance channels, whereas the Jordan block directly provides \(d\cos(\omega d)\) and \(d\sin(\omega d)\)-type modes.

\begin{figure*}[t]
    \centering
    \includegraphics[width=0.86\linewidth,trim=0 45 0 0,clip]{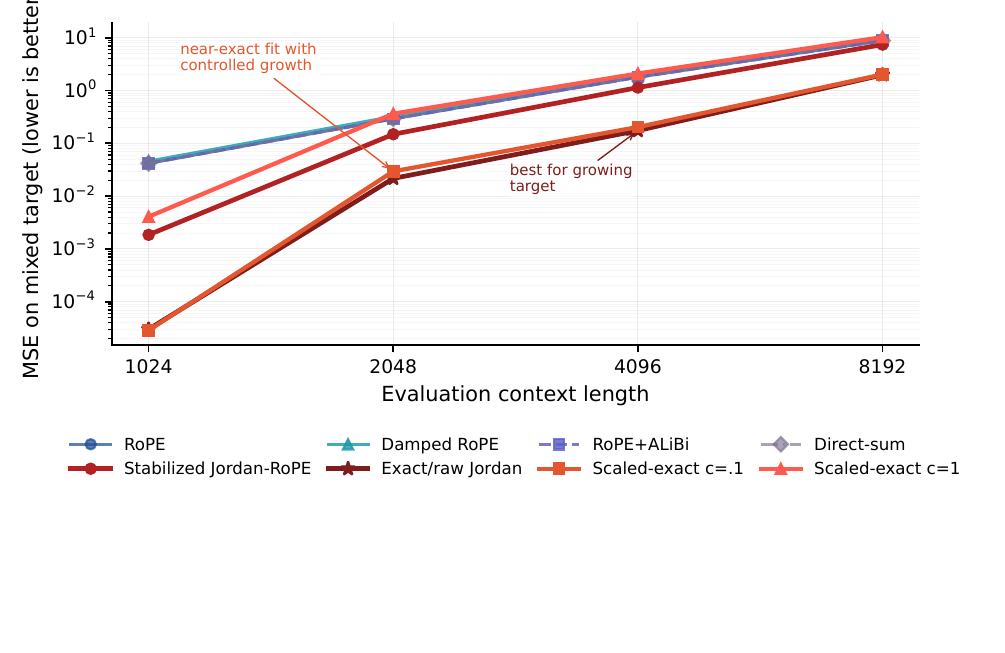}
    \caption{Jordan features directly fit mixed distance-phase targets. On \(y(d)=(d/L)\cos(\omega d)\), Exact/raw Jordan and Scaled-exact \(c=.1\) extrapolate substantially better than RoPE, RoPE+ALiBi, and direct-sum baselines at the primitive fixed-basis level.}
    \label{fig:basis-mixed}
\end{figure*}

Table~\ref{tab:toy-summary} shows the expected specialization.
RoPE solves the phase target, ALiBi/direct distance solves the linear target, and the Jordan-family features are strongest on the mixed target.
The large Linear MSE of the Jordan variants is expected: the complex Jordan block is not a pure distance basis unless an \(\omega=0\) block is included.
Its advantage is on distance-modulated phase targets, not isolated linear-distance targets.
The strongest gains in this diagnostic come from the Exact/raw Jordan and Scaled-exact \(c=0.1\) bases.
Exact/raw Jordan has the lowest mixed MSE because the target itself grows linearly.
The Scaled-exact basis with \(c=0.1\) nearly matches Exact/raw Jordan on the mixed target, while \(c=1\) is over-damped for this growing synthetic target.
Stabilized Jordan-RoPE gives a smaller but consistent improvement, reflecting the tradeoff introduced by \(\tau(d)\).

\begin{table}[!htbp]
    \centering
    \caption{Basis-level MSE at evaluation length \(8192\). The mixed target isolates the coupled frequency-tangent feature; it does not test arbitrary downstream multiplication by an MLP.}
    \label{tab:toy-summary}
    \scriptsize
    \resizebox{\linewidth}{!}{\begin{tabular}{lccc}
\toprule
Method & Phase MSE & Linear MSE & \textbf{Mixed MSE} \\
\midrule
RoPE & \textbf{0.000} & 17.618 & 8.791 \\
Damped RoPE & 0.054 & 17.583 & 9.510 \\
RoPE+ALiBi & 0.000 & 0.000 & 8.791 \\
ALiBi & 0.505 & \textbf{0.000} & 10.666 \\
Direct-sum & 0.000 & 0.000 & 8.819 \\
Stabilized Jordan-RoPE & 0.039 & 50.463 & 7.414 \\
Exact/raw Jordan & 0.004 & 462.392 & \textbf{1.975} \\
Scaled-exact c=.1 & 0.011 & 541.274 & \textbf{2.011} \\
Scaled-exact c=1 & 0.309 & 20.485 & 10.254 \\
\bottomrule
\end{tabular}
}
\end{table}

\subsection{Structured sequence probes}

We also evaluate fixed relative bases on structured targets designed to contain phase drift, seasonal trends, damped waves, rhythm envelopes, and motif spacing.
These are still basis-level probes, but they test a broader set of distance-modulated oscillatory patterns.

The structured probes support the same mechanism-level picture.
When the target contains a persistent product of distance and phase, Jordan modes are favored.
When the target is pure phase, RoPE or Jordan \((\gamma=0)\) variants can match it more cleanly.
Full per-method probe scores are reported in Appendix~\ref{app:structured-full}.

\begin{table}[!htbp]
    \centering
    \caption{Structured fixed-basis probes confirm the same mechanism-level pattern: when the target contains persistent distance-modulated oscillation, Exact/raw Jordan modes are favored.}
    \label{tab:structured-summary}
    \scriptsize
    \resizebox{\linewidth}{!}{\begin{tabular}{llc}
\toprule
Benchmark & Best method & MSE@8192 \\
\midrule
Phase drift & \textbf{Jordan ($\gamma=0$)} & 0.003 \\
Seasonal trend & \textbf{Exact/raw Jordan} & 1.686 \\
Damped wave & \textbf{Exact/raw Jordan} & 0.000 \\
Rhythm envelope & \textbf{Exact/raw Jordan} & 0.835 \\
Motif spacing & \textbf{Exact/raw Jordan} & 0.229 \\
\bottomrule
\end{tabular}
}
\end{table}

\begin{figure*}[t]
    \centering
    \includegraphics[width=0.86\linewidth,trim=0 45 0 0,clip]{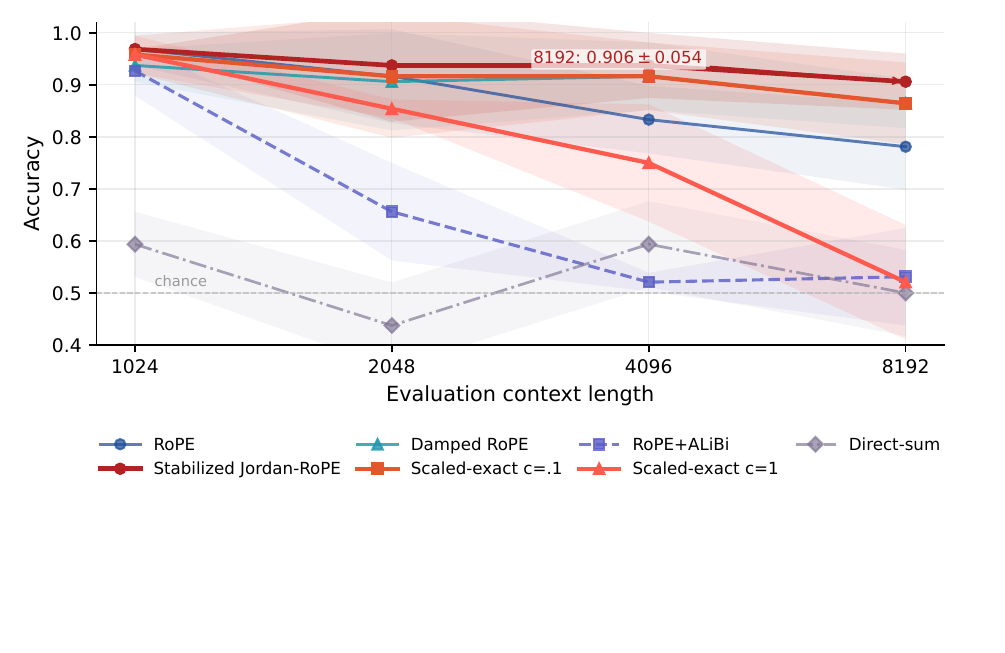}
    \caption{A Transformer can use the coupled Jordan mode when the teacher kernel requires it. On the synthetic query-LM task with \(K(d)=(d/L)\cos(\omega d)\), Stabilized Jordan-RoPE preserves the highest long-context accuracy, while direct-sum and RoPE+ALiBi fail to reliably recover the coupled rule under the same training budget.}
    \label{fig:kernel-lm}
\end{figure*}

\subsection{Jordan-friendly synthetic LM}

The strongest discrete sequence evidence comes from a synthetic query language-model task.
Each sequence contains random bits followed by a query token.
The next-token label is determined by the sign of a weighted sum of previous bits using the teacher kernel
\[
K(d)=\frac{d}{L}\cos(\omega d),
\]
with \(L=1024\).
The model must learn a distance-modulated phase rule and extrapolate it beyond the training length.

\begin{table*}[t]
    \centering
    \caption{Synthetic query-LM endpoint summary after 1200 training steps. The 8192-context endpoint is the main extrapolation test.}
    \label{tab:kernel-summary}
    \small
    \begin{tabularx}{\linewidth}{lccX}
\toprule
Method & Acc@1024 & Acc@8192 & Interpretation \\
\midrule
RoPE & 0.969 $\pm$ 0.031 & 0.781 $\pm$ 0.083 & phase-only baseline \\
Damped RoPE & 0.938 $\pm$ 0.031 & 0.865 $\pm$ 0.048 & strong damping baseline \\
RoPE+ALiBi & 0.927 $\pm$ 0.048 & 0.531 $\pm$ 0.094 & recency composition baseline \\
Direct-sum & 0.594 $\pm$ 0.062 & 0.500 $\pm$ 0.083 & fails coupled rule \\
Stabilized Jordan-RoPE & 0.969 $\pm$ 0.000 & \textbf{0.906 $\pm$ 0.054} & best long-context \\
Scaled-exact c=.1 & 0.958 $\pm$ 0.036 & 0.865 $\pm$ 0.079 & controlled exact variant \\
Scaled-exact c=1 & 0.958 $\pm$ 0.036 & 0.521 $\pm$ 0.110 & over-damped for this target \\
\bottomrule
\end{tabularx}

\end{table*}

At length \(8192\), Stabilized Jordan-RoPE reaches \(0.906\pm0.054\) accuracy, compared with \(0.865\pm0.048\) for damped RoPE, \(0.781\pm0.083\) for RoPE, and \(0.500\pm0.083\) for direct-sum.
The Scaled-exact variant with \(c=0.1\) reaches \(0.865\pm0.079\), roughly matching damped RoPE but below Stabilized Jordan-RoPE on this task.
With \(c=1\), accuracy drops to \(0.521\pm0.110\), indicating that the teacher kernel is too growth-dependent for strong damping.
This is the main evidence that Stabilized Jordan-RoPE can be useful in a trainable sequence model when the task actually rewards the coupled mode.
These sequence results should be interpreted as an inductive-bias comparison under a fixed architecture and training budget, not as an approximation lower bound.
Direct-sum and additive baselines may form coupled products through deeper nonlinear computation; the point is that the Jordan block provides the coupled distance-phase feature directly at the pre-softmax positional-bilinear level.

\subsection{High-order fixed-basis probes}

We test the frequency-jet hierarchy from Section~\ref{sec:high-order} with fixed-basis high-jet probes.
Let \(x=d/L\), with \(L\) the training length.
The scaled targets are
\[
y_r(d)=x^r\expe^{-0.1x}\cos(\omega d),
\]
fit on \(d<L\) and evaluated at length \(8192\).
The comparison includes frequency-count matched RoPE controls, direct-sum polynomial controls, ordinary order-two Jordan-RoPE, order-three and order-four Stabilized Jordan-RoPE variants, and Scaled-exact high-jet bases.
Table~\ref{tab:high-jet-basis} also reports the theoretical minimal Jordan order for each jet target and an explicit scaled \(m=2\) \(R^2\) check.
The control column reports the best \(R^2\) among the matched RoPE and direct-sum controls.

\begin{table*}[t]
    \centering
    \caption{High-order fixed-basis probes verify the finite-jet hierarchy. A target with jet order \(r\) requires a Jordan chain of order \(m\ge r+1\); the control column shows that matched RoPE and direct-sum polynomial controls do not supply the same local frequency jet. The first row uses an undamped first-jet target; Appendix~\ref{app:matched-high-jet} gives the damping-matched control where \(m=2\) fits the first jet essentially exactly.}
    \label{tab:high-jet-basis}
    \small
    \resizebox{\linewidth}{!}{\begin{tabular}{llclccc}
\toprule
Target at length 8192 & Min. order & Scaled \(m=2\) \(R^2\) & Best high-jet basis & MSE & \(R^2\) & Best control \(R^2\) \\
\midrule
\(x\cos(\omega d)\) (undamped) & \textbf{2} & 0.8387 & \textbf{Scaled-exact \(m=4,c=.1\)} & \(8.14{\times}10^{-3}\) & \textbf{0.9992} & \textcolor{gray}{0.1757} \\
\(x^2 e^{-.1x}\cos(\omega d)\) & \textbf{3} & 0.2908 & \textbf{Scaled-exact \(m=3,c=.1\)} & \(3.14{\times}10^{-6}\) & \textbf{1.0000} & \textcolor{gray}{0.0326} \\
\(x^3 e^{-.1x}\cos(\omega d)\) & \textbf{4} & 0.0396 & \textbf{Scaled-exact \(m=4,c=.1\)} & \(1.61\) & \textbf{0.9997} & \textcolor{gray}{0.0033} \\
\bottomrule
\end{tabular}
}
\end{table*}

The fixed-basis results cleanly support the jet interpretation.
Matched RoPE and direct-sum controls have low \(R^2\) on the scaled second- and third-jet targets, while the Scaled-exact order-three and order-four Jordan bases fit them almost exactly.
The selection of an order-four basis on the first-jet target should not be read as evidence that a fourth-order chain is necessary for the first jet.
The minimal order for a first jet is two.
Here the first target is the undamped \(x\cos(\omega d)\), while the scaled bases include the damping factor \(\expe^{-0.1x}\); the higher-order scaled bases can use their nested lower jets and extra polynomial terms to compensate this finite-grid damping mismatch.
Appendix~\ref{app:matched-high-jet} reports the matched scaled target \(x\expe^{-0.1x}\cos(\omega d)\), where the expected \(m=2\) basis fits the first jet essentially exactly.
This is a basis-level result: it shows that the intended functions are present in the positional feature family, not that a trained Transformer will necessarily learn to use every higher jet.

\subsection{High-jet synthetic query-LM}

We therefore also run a high-jet synthetic query-LM task.
As in the first-jet synthetic LM, each sequence contains random bits followed by a query token, and the label is determined by the sign of a teacher-kernel weighted sum.
Here the teacher kernels are the scaled second- and third-jet targets \(y_2(d)\) and \(y_3(d)\).
Appendix~\ref{app:high-jet-lm} reports the long-context results most relevant to the high-order claim.
The second-jet task favors the order-three exact variant.
For the third-jet task, allowing the nilpotent amplitude to reach order-one scale makes the order-four exact variant the best mean-loss method at contexts 512 and 1024, and the best accuracy method at 1024.

These results should be read as a controlled extension of the main order-two claim.
They do not imply that higher-order Jordan blocks are broadly better positional encodings.
Rather, they show that when the teacher actually contains a higher frequency jet, the corresponding nilpotent chain can provide a usable inductive bias.
The detailed high-jet synthetic matrix is reported in Appendix~\ref{app:high-jet-lm}.

\subsection{Task-selective music-token transfer: a MusicNet positive case}

Music is a natural stress test for relative-position mechanisms because repetition, motif return, phrase-level timing, and long-range self-reference occur on multiple timescales.
This motivation also underlies Music Transformer, where relative attention was used to model long-term musical structure \citep{huang2018musictransformer}.
We therefore include a small music-token transfer sweep using byte-level streams derived from MAESTRO MIDI, FMA-small mel features, MusicNet mel features, and a mixed audio/MIDI corpus \citep{hawthorne2019maestro,defferrard2017fma,thickstun2017musicnet}.
Models are trained at length \(256\) for 800 steps and evaluated up to length \(2048\) with five seeds.
Table~\ref{tab:music-high-jet} reports the 2048-context validation loss.

\begin{table*}[t]
    \centering
    \caption{Music/audio transfer gives a selective natural positive case, not a universal win. MusicNet is the strongest case for Scaled-exact \(m=3\); other datasets honestly favor damping or direct-sum structure.}
    \label{tab:music-high-jet}
    \small
    \begin{tabularx}{\linewidth}{
  >{\raggedright\arraybackslash}p{0.14\linewidth}
  >{\raggedright\arraybackslash}p{0.30\linewidth}
  >{\raggedright\arraybackslash}p{0.21\linewidth}
  >{\raggedright\arraybackslash}X}
\toprule
Dataset & Winner at 2048 & Scaled-exact \(m=3\) status & Interpretation \\
\midrule
Mixed audio/MIDI & Direct-sum, \(4.1369\pm0.1200\) & close, \(4.2336\pm0.0552\) & separate channels suffice \\
MAESTRO & Scaled-exact \(m=3,c=.1,\eta_0=.2\), \(3.6857\pm0.2822\) & \textbf{best} & weak high-jet signal \\
FMA-small & Damped RoPE, \(3.8711\pm0.0297\) & not best, \(4.0991\pm0.2103\) & damping dominates \\
\textbf{MusicNet} & \textbf{Scaled-exact \(m=3,c=.5,\eta_0=.1\), \(3.5901\pm0.1201\)} & \textbf{best} & strongest natural positive case \\
\bottomrule
\end{tabularx}

\end{table*}

The result is not uniformly positive across datasets, but it reveals a clear natural positive case.
On MusicNet mel-byte tokens, the Scaled-exact order-three variant \((m=3,c=0.5,\eta_0=0.1)\) is the best method at contexts 512, 1024, and 2048, reaching \(3.5901\pm0.1201\) validation loss at length 2048.
This compares with \(4.1211\pm0.3638\) for damped RoPE and \(4.7015\pm0.2265\) for the direct-sum phase/distance baseline.
By contrast, the mixed audio/MIDI corpus continues to favor direct-sum, FMA-small favors damped RoPE, and MAESTRO gives only a weaker Scaled-exact \(m=3\) signal.
Thus the MusicNet result should be read as evidence for task-selective transfer of the high-jet inductive bias, not as a broad claim that higher-order Jordan blocks are uniformly better for audio.

We further inspected the MusicNet checkpoints by measuring the positioned query/key norm profiles.
For the winning \(m=3,c=0.5,\eta_0=0.1\) model, the positioned-key norm grows by roughly \(206\times\) from position 0 to 2047, while the positioned-query norm shrinks to roughly \(1\%\) of its initial value.
The learned parameters have mean \(\gamma\approx 0.0024\) and mean \(|\eta|\approx 0.104\).
This norm profile is not merely a numerical artifact: it is the expected signature of an exact non-orthogonal Jordan factorization, where the key side stores a position-amplified high-jet basis and the contragredient query side compensates it in the bilinear score.
It also marks a practical caveat for cache compression, because the same steep \(K\)-cache norm profile may make low-bit key quantization harder.
Overall, the music-token results reinforce the task-selective interpretation: high-order Jordan modes are useful when the data rewards higher-order distance-modulated phase structure, but they are not a uniform replacement for damping or direct-sum biases.

\subsection{WikiText boundary tests}
\label{sec:wikitext}

Finally, we train a small causal Transformer byte language model on WikiText-103.
The run uses \(d_{\mathrm{model}} = 384\), six layers, six heads, training length \(1024\), and evaluation lengths up to \(8192\).
This is a small natural-language test case, not a claim that the Jordan family dominates established long-context biases.

\begin{figure*}[t]
    \centering
    \includegraphics[width=0.86\linewidth,trim=0 45 0 0,clip]{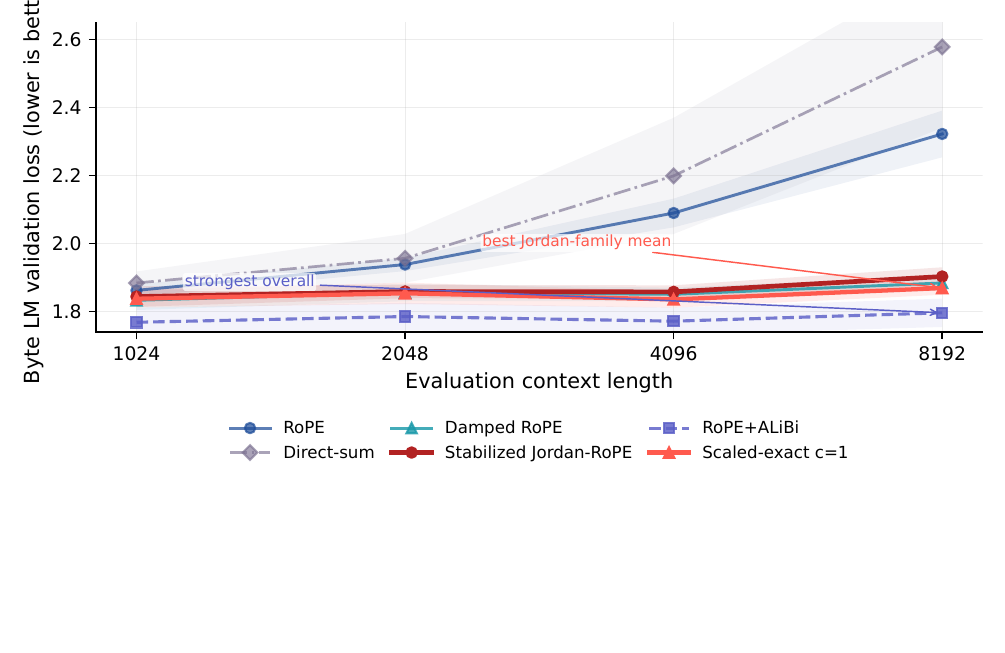}
    \caption{WikiText is a boundary test rather than a broad win. RoPE+ALiBi remains strongest overall, while Scaled-exact Jordan-RoPE with \(c=1\) gives the best mean loss within the Jordan family and stays competitive with damped RoPE.}
    \label{fig:wikitext}
\end{figure*}

\begin{table*}[t]
    \centering
    \caption{WikiText-103 byte LM boundary summary at length \(8192\). \(\Delta\) loss is relative to each method's loss at length \(1024\); the larger-\(\eta_0\) sweep is moved to Appendix~\ref{app:wikitext-eta}.}
    \label{tab:wikitext-summary}
    \small
    \begin{tabularx}{\linewidth}{lcccX}
\toprule
Method & Loss@8192 & $\Delta$ loss & $|\eta|$ mean & Role \\
\midrule
RoPE & 2.322 $\pm$ 0.069 & 0.460 $\pm$ 0.056 & -- & phase-only baseline \\
Damped RoPE & 1.884 $\pm$ 0.022 & 0.051 $\pm$ 0.010 & 0.0000 & damping baseline \\
RoPE+ALiBi & \textbf{1.796 $\pm$ 0.042} & 0.028 $\pm$ 0.007 & -- & best overall \\
ALiBi & 1.932 $\pm$ 0.028 & 0.017 $\pm$ 0.002 & -- & additive recency \\
Direct-sum & 2.578 $\pm$ 0.238 & 0.694 $\pm$ 0.206 & 0.0013 $\pm$ 0.0000 & weak baseline \\
Stabilized Jordan-RoPE & 1.903 $\pm$ 0.027 & 0.059 $\pm$ 0.003 & 0.0013 $\pm$ 0.0000 & bounded-$\tau$ default \\
Scaled-exact c=1 & \textbf{1.869 $\pm$ 0.020} & 0.031 $\pm$ 0.009 & 0.0014 $\pm$ 0.0001 & best Jordan family \\
Scaled-exact c=.1 & 1.907 $\pm$ 0.027 & 0.047 $\pm$ 0.010 & 0.0021 $\pm$ 0.0001 & weaker damping \\
\bottomrule
\end{tabularx}

\end{table*}

Table~\ref{tab:wikitext-summary} shows the boundary of the claim.
Stabilized Jordan-RoPE improves substantially over RoPE and direct-sum in long-context loss.
The Scaled-exact Jordan-RoPE variant with \(c=1\) gives the lowest Jordan-family mean loss, reaching \(1.869\pm0.020\) at length \(8192\), compared with \(1.903\pm0.027\) for Stabilized Jordan-RoPE and \(1.884\pm0.022\) for damped RoPE.
The difference from damped RoPE is small at three seeds and should be read as a mean trend rather than a statistically decisive separation.
However, RoPE+ALiBi remains strongest at \(1.796\pm0.042\).
The learned default shear stays small, and larger \(\eta_0\) initializations for Stabilized Jordan-RoPE worsen long-length loss.
The natural-LM results should therefore not be read as evidence that large learned shear is broadly beneficial.
In this small byte-level setting, the learned shear remains small, larger forced \(\eta_0\) in Stabilized Jordan-RoPE hurts long-length loss, and the strongest gains appear to come from controlled damping with a compatible weak Jordan-RoPE channel.
Thus the WikiText result is best interpreted as evidence that a Scaled-exact Jordan-RoPE parameterization can coexist with useful damping, rather than as evidence for strong oscillatory-polynomial shear in natural language.

\subsection{High-jet WikiText auxiliary run}

We additionally test whether the high-order exact variants that helped in the synthetic probes and in selected music-token tasks transfer to WikiText.
Because order-three and order-four Jordan blocks require head dimensions divisible by \(6\) and \(8\), respectively, the auxiliary high-jet runs use a separate \(d_{\mathrm{model}} = 192\), three-layer, four-head byte LM trained at length \(512\) on a WikiText-103 subset.
It should therefore be read as an internal comparison within the auxiliary matrix, not as a direct numerical replacement for Table~\ref{tab:wikitext-summary}.

\begin{table*}[t]
    \centering
    \caption{Auxiliary high-jet WikiText summary. High-order jets do not replace recency bias by themselves, but weak normalized jet corrections can help when composed with ALiBi.}
    \label{tab:high-jet-wikitext-summary}
    \small
    \begin{tabularx}{\linewidth}{l l c X}
\toprule
Setting & Best method & Loss@8192 & Takeaway \\
\midrule
No ALiBi & Exact \(m=2,c=1.5,\eta_0=.010\) & \(2.0482\pm0.0040\) & High jets alone do not replace recency; normalized \(m=3\) is close but base-heavy. \\
+ALiBi & \textbf{Norm-jet \(m=3,c=1.5,\alpha\approx(.69,.21,.10)\)+ALiBi} & \textbf{\(1.9759\pm0.0138\)} & Weak jet corrections can improve on top of a strong recency prior. \\
Baseline & RoPE+ALiBi & \(2.0037\pm0.0167\) & Strong additive-bias baseline. \\
\bottomrule
\end{tabularx}

\end{table*}

Table~\ref{tab:high-jet-wikitext-summary} summarizes the auxiliary matrix.
Without ALiBi, the best standalone method remains an order-two exact control, and the best normalized order-three spectrum is close but strongly base-heavy, approximately \((.84,.11,.05)\).
When the same ALiBi bias is added to tuned Jordan candidates, a learned normalized order-three jet spectrum, approximately \(\alpha=(.69,.21,.10)\), reaches \(1.9759\pm0.0138\), slightly lower than the RoPE+ALiBi baseline in the same auxiliary setup.
The full no-ALiBi and ALiBi-composed tables are reported in Appendix~\ref{app:high-jet-wikitext}.
Together these results support a narrow conclusion: on this WikiText subset, high-order jets are most useful as weak corrections on top of a strong recency prior, not as standalone replacements for that prior.

\section{Discussion and Limitations}

The purpose of the experiments is to isolate the induced relative-position basis, not to establish a new state-of-the-art positional encoding.
The fixed-basis probes verify that the complex Jordan block supplies coupled oscillatory-polynomial modes directly.
The synthetic LM shows that a Transformer can use the Stabilized Jordan-RoPE channel when the task requires that mode.
The WikiText runs show partial transfer to natural language length extrapolation: Scaled-exact variants give the best mean losses within the Jordan family in the base comparison, and focused auxiliary runs show that normalized high-jet spectra do not replace ALiBi by themselves but can improve when composed with it.
The high-order experiments extend this picture from the first tangent mode \(d\,\expe^{i\omega d}\) to a finite frequency-jet hierarchy.
The Scaled-exact order-three MusicNet result suggests that such higher jets can be useful beyond teacher-designed synthetic kernels, but the mixed audio/MIDI, FMA-small, and no-ALiBi WikiText results also show that this benefit is task-selective.

\paragraph{Defective generators as non-normal attention dynamics.}
The original puzzle behind this construction was the status of defective generators in positional encoding.
Diagonalizable generators have a clear interpretation: real eigenvalues give exponential decay or growth, and complex conjugate eigenvalues give rotations.
Defective generators are less transparent because their exponentials contain polynomial terms.
The experiments here suggest one useful interpretation.
In the complex case, the nilpotent direction should not be viewed as an isolated polynomial bias; it is a tangent direction attached to the rotary phase.
This is why the natural primitive mode is \(d\,\expe^{i\omega d}\), not merely \(d\).
Equivalently, Jordan-RoPE is a non-reduced Fourier positional encoding: ordinary RoPE samples reduced points on the Fourier phase curve, while a complex Jordan block replaces each point by a finite infinitesimal neighborhood.
The high-order variants then supply the corresponding osculating jet rather than simply adding more Fourier frequencies.

From this angle, Jordan-RoPE is a small attention-kernel instance of non-normal linear dynamics, where stable eigenvalues need not preclude large transient responses \citep{trefethen1993hydrodynamic,trefethen2005spectra}.
Stabilized Jordan-RoPE and Scaled-exact Jordan-RoPE are practical ways of retaining this transient response while controlling long-lag growth.
This interpretation also explains the empirical boundary: tasks with explicit distance-modulated phase structure can benefit from the Jordan channel, while natural byte-level language modeling in our small setting mostly rewards damping and recency bias.

\paragraph{Linear dynamical kernels.}
The oscillatory-polynomial functions produced by a Jordan block are also familiar from linear dynamical systems: matrix exponentials of defective generators contain polynomial factors multiplying exponentials.
This connects Jordan-RoPE conceptually to the broader use of linear state-space dynamics and long convolutional sequence operators in long-range sequence modeling, including S4, DSS, and Hyena \citep{gu2021s4,gupta2022dss,poli2023hyena}.
However, our construction is not a state-space layer.
It uses the same exponential-polynomial algebra inside the relative attention logit kernel, thereby changing the primitive pairwise positional basis rather than replacing attention by a convolutional or recurrent operator.

There are several limitations.
First, bounded \(\tau\) sacrifices the exact one-parameter law.
Second, the current Transformer implementation is a position-wise stabilized factorization; it is not the same object as an explicit pairwise \(G_{\mathrm{stab}}(d)\) relative kernel once \(\tau\) is nonlinear.
Third, although the relative kernel is damped, the non-orthogonal position-wise contragredient factorization can introduce compensating absolute rescalings in queries and keys.
In our small-scale experiments this is numerically stable, but exact pairwise relative-kernel implementations or centered factorizations may be preferable for very long contexts.
Fourth, the natural-LM runs are small, byte-level, and undertrained relative to modern long-context language-model practice.
Fifth, direct-sum and additive baselines might produce coupled products through deeper nonlinear computation, so our expressivity statement is intentionally restricted to the primitive pre-softmax bilinear positional basis.
Sixth, retrieval-style tasks require separate training and tuning, and are not reported here.

Future work should test broader \(c\), \(\eta\), frequency-grid choices, and exact relative-kernel implementations, as well as tasks where distance-modulated phase structure occurs naturally.

\paragraph{Code and data availability.}
The implementation, configuration files, smoke tests, and lightweight reproducibility bundle are available at \url{https://github.com/ybzhang-nxu/jordan_rope}.
The \texttt{repro/} directory contains the CSV logs and scripts used to regenerate the reported tables and figures.
Large checkpoints, downloaded datasets, and processed corpora are not included because they can be regenerated from the released scripts and configurations.

\appendix

\section{Implementation Notes}

The experimental module exposes \texttt{JordanRoPE.apply(q, k, positions)} with tensors shaped \(B\times H\times T\times d_h\) and \(d_h\) divisible by four for the order-two realified block.
Default parameters use \(\omega_m=\theta^{-2m/d_h}\), \(\gamma=\mathrm{softplus}(a)+10^{-4}\), \(\eta=0.1\tanh(b)\), and \(L=1024\).
The default Stabilized Jordan-RoPE path uses bounded \(\tau\); the \texttt{jordan\_raw\_tau} ablation disables bounded \(\tau\).
Although the relative kernel is damped, the position-wise contragredient factorization uses compensating absolute query/key rescalings.
For example, the Scaled-exact variant with \(c=1\), \(L=1024\), and \(T=8192\) can introduce factors on the order of \(\expe^8\) before cancellation in the bilinear score.
This did not cause numerical failures in our runs, but it is a practical reason to consider centered factorizations or explicit pairwise relative kernels for much longer contexts.
A simple mitigation is to use a centered factorization over a finite evaluation window.
For an exact one-parameter representation,
\[
A(i-c_0)^{-1}A(j-c_0)=G(i-j)
\]
for any center \(c_0\).
Choosing \(c_0\) near the middle of the current window reduces the largest absolute query/key rescaling while preserving the exact relative score.
For nonlinear bounded \(\tau\), this no longer restores a pure lag law, but it can still reduce the numerical range.

\section{High-Jet Synthetic Query-LM Details}
\label{app:high-jet-lm}

Table~\ref{tab:high-jet-lm} reports the detailed high-jet synthetic query-LM matrix.
The second-jet run uses three seeds; the third-jet \(\eta_0\)-range confirmation uses five seeds.
Methods are trained at length \(256\) and evaluated at longer contexts.

\begin{table}[!htbp]
    \centering
    \caption{Trainable models can exploit higher frequency jets when the teacher kernel contains them. The second-jet task favors exact \(m=3\), while the third-jet task benefits from exact \(m=4\) once the nilpotent amplitude is allowed to be large enough.}
    \label{tab:high-jet-lm}
    \scriptsize
    \resizebox{\linewidth}{!}{\begin{tabular}{llcccc}
\toprule
Teacher kernel & Context & Best accuracy method & Accuracy & Best loss method & Loss \\
\midrule
Scaled second jet & 512 & Exact \(m=3,c=.1,\eta_0=.01\) & \(0.9097\pm0.0335\) & Exact \(m=3,c=.1,\eta_0=.01\) & \(0.2697\pm0.1507\) \\
\textbf{Scaled second jet} & \textbf{1024} & \textbf{Exact \(m=3,c=.1,\eta_0=.01\)} & \textbf{\(0.9410\pm0.0318\)} & \textbf{Exact \(m=3,c=.1,\eta_0=.01\)} & \textbf{\(0.2280\pm0.1652\)} \\
\midrule
Scaled third jet & 512 & Exact \(m=3,c=.1,\eta_0=.1\) & \(0.9333\pm0.0380\) & Exact \(m=4,c=.1,\eta_0=.1\) & \(0.1713\pm0.0747\) \\
\textbf{Scaled third jet} & \textbf{1024} & \textbf{Exact \(m=4,c=.1,\eta_0=.1\)} & \textbf{\(0.9083\pm0.0298\)} & \textbf{Exact \(m=4,c=.1,\eta_0=.1\)} & \textbf{\(0.2328\pm0.0652\)} \\
\bottomrule
\end{tabular}
}
\end{table}

\section{Full Structured Probe Scores}
\label{app:structured-full}

Table~\ref{tab:structured-full} reports all fixed-basis structured probe scores at evaluation length \(8192\).
In this appendix, ``Jordan \((\gamma=0)\)'' denotes the same Jordan basis with damping disabled, while ``Jordan \((m=3)\)'' uses an order-three nilpotent chain.
In the kernel and WikiText tables, ``large \(\eta_0\) init'' denotes the larger initial-shear variant used in the \(\eta_0\) initialization sweep.

\begin{table}[!htbp]
    \centering
    \caption{Full structured probe MSE at length \(8192\).}
    \label{tab:structured-full}
    \scriptsize
    \resizebox{\linewidth}{!}{\begin{tabular}{lrrrrr}
\toprule
Method & Phase drift & Seasonal trend & Damped wave & Rhythm envelope & Motif spacing \\
\midrule
RoPE & 0.004 & 7.120 & 1.345 & 3.449 & 0.693 \\
ALiBi & 0.505 & 9.996 & 2.103 & 4.671 & 1.093 \\
RoPE+ALiBi & 0.004 & 7.121 & 1.346 & 3.450 & 0.696 \\
Damped RoPE & 0.056 & 8.164 & 1.493 & 3.901 & 0.829 \\
Direct-sum & 0.004 & 7.145 & 1.345 & 3.457 & 0.711 \\
Stabilized Jordan-RoPE & 0.041 & 6.395 & 0.935 & 3.034 & 0.511 \\
Jordan ($\gamma=0$) & \textbf{0.003} & 4.632 & 0.608 & 2.226 & 0.265 \\
Exact/raw Jordan & 0.005 & \textbf{1.686} & \textbf{0.000} & \textbf{0.835} & \textbf{0.229} \\
Jordan ($m=3$) & 0.028 & 4.832 & 0.543 & 2.302 & 0.259 \\
\bottomrule
\end{tabular}
}
\end{table}

\section{WikiText Initialization and High-Jet Details}
\label{app:wikitext-eta}
\label{app:high-jet-wikitext}

Table~\ref{tab:wikitext-eta-sweep} reports the Stabilized Jordan-RoPE \(\eta_0\) initialization sweep omitted from the main WikiText boundary table.
Larger forced \(\eta_0\) increases \(|\eta|\) and worsens long-context loss, supporting the main text's weak-shear interpretation.

\begin{table}[!htbp]
    \centering
    \caption{WikiText Stabilized Jordan-RoPE \(\eta_0\) initialization sweep at length \(8192\). Larger forced \(\eta_0\) hurts this natural-LM setting.}
    \label{tab:wikitext-eta-sweep}
    \scriptsize
    \resizebox{\linewidth}{!}{\begin{tabular}{lccc}
\toprule
Initialization & Loss@8192 & \(\Delta\) loss & \(|\eta|\) mean \\
\midrule
Stabilized Jordan-RoPE \(\eta_0=.005\) & \(1.945\pm0.055\) & \(0.074\pm0.002\) & \(0.0040\pm0.0000\) \\
Stabilized Jordan-RoPE \(\eta_0=.010\) & \(1.970\pm0.012\) & \(0.097\pm0.019\) & \(0.0081\pm0.0000\) \\
Stabilized Jordan-RoPE \(\eta_0=.020\) & \(2.131\pm0.049\) & \(0.177\pm0.050\) & \(0.0176\pm0.0001\) \\
Stabilized Jordan-RoPE \(\eta_0=.050\) & \(2.275\pm0.044\) & \(0.219\pm0.035\) & \(0.0476\pm0.0001\) \\
\bottomrule
\end{tabular}
}
\end{table}

Tables~\ref{tab:high-jet-wikitext}, \ref{tab:high-jet-wikitext-norm-noalibi}, and \ref{tab:high-jet-wikitext-alibi} give the auxiliary high-jet WikiText matrices summarized in Table~\ref{tab:high-jet-wikitext-summary}.

\begin{table}[!htbp]
    \centering
    \caption{Auxiliary high-jet WikiText-103 subset run at length \(8192\), without adding ALiBi to the Jordan variants. Exact high-jet variants with \(c=1\) form a strong second tier, but RoPE+ALiBi remains best in this no-composition comparison.}
    \label{tab:high-jet-wikitext}
    \scriptsize
    \resizebox{\linewidth}{!}{\begin{tabular}{lcc}
\toprule
Method & Loss@8192 & $\Delta$ loss \\
\midrule
RoPE+ALiBi & \(2.1316\pm0.0126\) & \(-0.1631\pm0.0274\) \\
Exact \(m=3,c=1,\eta_0=.1\) & \(2.1734\pm0.0106\) & \(-0.1725\pm0.0122\) \\
Exact \(m=2,c=1,\eta_0=.002\) & \(2.1740\pm0.0145\) & \(-0.1845\pm0.0174\) \\
Exact \(m=4,c=1,\eta_0=.1\) & \(2.1933\pm0.0046\) & \(-0.1663\pm0.0144\) \\
Exact \(m=3,c=.1,\eta_0=.1\) & \(2.1945\pm0.0082\) & \(-0.1586\pm0.0359\) \\
Exact \(m=4,c=.1,\eta_0=.1\) & \(2.1994\pm0.0142\) & \(-0.1659\pm0.0164\) \\
Damped RoPE & \(2.3659\pm0.0244\) & \(-0.0118\pm0.0417\) \\
Stabilized Jordan-RoPE \(m=3\) & \(2.3799\pm0.0255\) & \(0.0132\pm0.0450\) \\
Stabilized Jordan-RoPE & \(2.3876\pm0.0127\) & \(0.0354\pm0.0417\) \\
RoPE & \(2.5184\pm0.0674\) & \(0.1589\pm0.0810\) \\
Direct-sum & \(2.5204\pm0.0156\) & \(0.1595\pm0.0425\) \\
Stabilized Jordan-RoPE \(m=4\) & \(2.5226\pm0.0682\) & \(0.1709\pm0.0523\) \\
\bottomrule
\end{tabular}
}
\end{table}

\begin{table}[!htbp]
    \centering
    \caption{Auxiliary normalized jet-spectrum WikiText-103 subset run at length \(8192\), without ALiBi. Normalization makes the coefficients a spectrum over jet orders rather than an unnormalized gain.}
    \label{tab:high-jet-wikitext-norm-noalibi}
    \scriptsize
    \resizebox{\linewidth}{!}{\begin{tabular}{lcc}
\toprule
Method & Loss@8192 & Notes \\
\midrule
Exact \(m=2,c=1.5,\eta_0=.010\) & \(2.0482\pm0.0040\) & best mean \\
Norm-jet \(m=3,\alpha\approx(.84,.11,.05)\) & \(2.0535\pm0.0091\) & learned, base-heavy \\
Norm-jet \(m=3,\alpha\approx(.69,.21,.10)\) & \(2.0590\pm0.0076\) & learned \\
Jetmix \(m=3,\alpha=(1,.10,.05)\) & \(2.0605\pm0.0029\) & unnormalized \\
Norm-jet fixed \(m=3,\alpha=(.70,.20,.10)\) & \(2.0607\pm0.0035\) & fixed simplex \\
Norm-jet fixed \(m=3,\alpha=(.50,.30,.20)\) & \(2.0622\pm0.0079\) & larger high-jet mass \\
Norm-jet \(m=4,\alpha\approx(.38,.31,.20,.10)\) & \(2.0827\pm0.0056\) & learned order-four \\
RoPE & \(2.6392\pm0.1091\) & no ALiBi \\
\bottomrule
\end{tabular}
}
\end{table}

\begin{table}[!htbp]
    \centering
    \caption{ALiBi-composed auxiliary WikiText run at length \(8192\). All rows use the same ALiBi bias.}
    \label{tab:high-jet-wikitext-alibi}
    \scriptsize
    \resizebox{\linewidth}{!}{\begin{tabular}{lcc}
\toprule
Method & Loss@8192 & Notes \\
\midrule
Norm-jet \(m=3,c=1.5,\alpha\approx(.69,.21,.10)\)+ALiBi & \(1.9759\pm0.0138\) & best mean \\
Jetmix \(m=3,c=1.5,\alpha=(1,.10,.05)\)+ALiBi & \(1.9805\pm0.0222\) & unnormalized \\
Exact \(m=2,c=1.5,\eta_0=.010\)+ALiBi & \(1.9812\pm0.0170\) & order-two control \\
Norm-jet fixed \(m=3,c=1.5,\alpha=(.50,.30,.20)\)+ALiBi & \(1.9822\pm0.0141\) & fixed simplex \\
Jetmixfull \(m=4,c=1.5,\alpha=(1,.50,.30,.20)\)+ALiBi & \(1.9877\pm0.0174\) & larger high-jet tail \\
RoPE+ALiBi & \(2.0037\pm0.0167\) & additive-bias baseline \\
\bottomrule
\end{tabular}
}
\end{table}

\section{Matched High-Jet Basis Control}
\label{app:matched-high-jet}

Table~\ref{tab:high-jet-matched-basis} repeats the fixed-basis high-jet check with damping matched across all targets:
\[
x^r\expe^{-0.1x}\cos(\omega d),\qquad r=1,2,3.
\]
This removes the finite-grid damping mismatch in the first row of Table~\ref{tab:high-jet-basis}.
With the matched first-jet target, the scaled \(m=2\) basis fits essentially exactly, while the second and third targets are fit by the expected \(m=3\) and \(m=4\) bases.

\begin{table}[!htbp]
    \centering
    \caption{Matched scaled high-jet fixed-basis controls at evaluation length \(8192\). The control column is the best \(R^2\) among matched RoPE and direct-sum polynomial controls.}
    \label{tab:high-jet-matched-basis}
    \scriptsize
    \resizebox{\linewidth}{!}{\begin{tabular}{lccccc}
\toprule
Matched scaled target & Min. order & Scaled \(m=2\) \(R^2\) & Scaled \(m=3\) \(R^2\) & Scaled \(m=4\) \(R^2\) & Best control \(R^2\) \\
\midrule
\(x e^{-.1x}\cos(\omega d)\) & 2 & 1.0000 & 1.0000 & 0.9989 & 0.2979 \\
\(x^2 e^{-.1x}\cos(\omega d)\) & 3 & 0.2908 & 1.0000 & 0.9995 & 0.0326 \\
\(x^3 e^{-.1x}\cos(\omega d)\) & 4 & 0.0396 & 0.3740 & 0.9997 & 0.0033 \\
\bottomrule
\end{tabular}
}
\end{table}

\FloatBarrier

\section{Scope of Claims}

The paper makes four deliberately narrow claims.
First, exact complex Jordan blocks are valid non-semisimple one-parameter positional representations when implemented with the contragredient query action.
Second, their primitive logit basis contains coupled modes \(d\cos(\omega d)\) and \(d\sin(\omega d)\).
Third, Stabilized Jordan-RoPE can help when a sequence task rewards those modes.
Fourth, order-\(m\) complex Jordan chains realize a finite frequency-jet basis \(\{d^r\expe^{i\omega d}\}_{r=0}^{m-1}\), and controlled high-jet probes show that the corresponding higher nilpotent chains can be useful when the target actually contains higher distance-modulated phase structure.
The MusicNet result and the ALiBi-composed WikiText jet-spectrum run are reported only as task-selective transfer evidence.
The normalized no-ALiBi WikiText ablation is a negative control: it shows that high-order jet spectra alone do not replace additive recency biases in this setting.
The composed result instead shows a small positive case for weak Jordan/jet corrections added on top of such a bias.
The paper does not claim that Stabilized Jordan-RoPE is an exact representation, that nilpotent positional encodings are new, or that the Jordan family dominates all existing positional encodings on natural language modeling.

\end{document}